\documentclass{article}

\usepackage[preprint]{nips_2018}

\usepackage[utf8]{inputenc} 
\usepackage[T1]{fontenc}    
\usepackage{hyperref}       
\usepackage{url}            
\usepackage{booktabs}       
\usepackage{amsfonts}       
\usepackage{nicefrac}       
\usepackage{microtype}      

\usepackage{graphicx}
\usepackage{algorithm}
\usepackage{algorithmic}

\usepackage[title]{appendix}
\usepackage{afterpage}
\usepackage{amsfonts}
\usepackage{amsmath}
\usepackage{amsthm}
\usepackage{amssymb}
\usepackage{dsfont}
\usepackage{enumitem}
\usepackage{filecontents}
\usepackage{todonotes}
\usepackage{graphicx}
\usepackage{subcaption}
\usepackage{wrapfig}
\usepackage{bbm}

\newtheorem{lemma}{Lemma}

\newcommand{\1}[1]{\mathds{1}{\left\{ #1 \right\}}}

\newcommand{\Set}[1]{\mathchoice{\left\{ #1 \right\}}{\{ #1 \}}{\{ #1 \}}{\{ #1 \}}}

\title{Adaptively Pruning Features for Boosted Decision Trees}
\author{
  Maryam Aziz \\
  Northeastern University \\
  Boston, Massachusetts, USA \\
  \texttt{azizm@ccs.neu.edu} \\
  \And
  Jesse Anderton \\
  Northeastern University \\
  Boston, Massachusetts, USA \\
  \texttt{jesse@ccs.neu.edu} \\
  \And
  Javed Aslam \\
  Northeastern University \\
  Boston, Massachusetts, USA \\
  \texttt{jaa@ccs.neu.edu} \\
}

\begin{document}

\maketitle

\begin{abstract}

Boosted decision trees enjoy popularity in a variety of applications;
however, for large-scale datasets, the cost of training a decision
tree in each round can be prohibitively expensive.  Inspired by ideas
from the multi-arm bandit literature, we develop a highly efficient
algorithm for computing exact greedy-optimal decision trees,
outperforming the state-of-the-art \texttt{Quick Boost} method.  We
further develop a framework for deriving lower bounds on the problem
that applies to a wide family of conceivable algorithms for the task
(including our algorithm and \texttt{Quick Boost}), and we demonstrate
empirically on a wide variety of data sets that our algorithm is
near-optimal within this family of algorithms.  We also
derive a lower bound applicable to any algorithm solving the task, and
we demonstrate that our algorithm empirically achieves performance
close to this best-achievable lower bound.

\end{abstract}

\section{Introduction}\label{introduction}

Boosting algorithms are among the most popular classification
algorithms in use today, e.g. in computer vision, learning-to-rank,
and text classification.  Boosting, originally introduced by
\citet{Schapire90thestrength, Freund:1995:BWL:220262.220446,
  Freund:1996:ENB:3091696.3091715}, is a family of machine learning
algorithms in which an accurate classification strategy is learned by
combining many ``weak'' hypotheses, each trained with respect to a
different weighted distribution over the training data.  These
hypotheses are learned sequentially, and at each iteration of boosting
the learner is biased towards correctly classifying the examples which
were most difficult to classify by the preceding weak hypotheses.

Decision trees \citep{Quinlan:1993:CPM:152181}, due to their simplicity
and representation power, are among the most popular weak learners
used in Boosting algorithms \citep{Freund:1996:ENB:3091696.3091715,
  Quinlan:1996:BBC:1892875.1892983}.  However, for large-scale data
sets, training decision trees across potentially hundreds of rounds of
boosting can be prohibitively expensive.  Two approaches to ameliorate
this cost include
(1) \emph{approximate decision tree training},
which aims to identify a subset of the
features and/or a subset of the training examples such that
\emph{exact} training on this subset yields a high-quality decision
tree,
and (2)
\emph{efficient exact decision tree training},
which aims to compute the greedy optimal decision tree over the entire data
set and feature space as efficiently as possible.
These two approaches complement each other:
approximate training often devolves to exact
training on a subset of the data.

  As such, we
consider the task of efficient exact decision tree learning in the
context of boosting where our primary objective is to minimize the
number of examples that must be examined for any feature in order to
perform greedy-optimal decision tree training. 
Our method
is simple to implement, and gains in feature-example efficiency directly corresponds to improvements in computation time.

The main contributions of the paper are as follows:\vspace{-0.5\baselineskip}
\begin{itemize}
\item We develop a highly efficient algorithm for computing exact
  greedy-optimal decision trees, \texttt{Adaptive-Pruning Boost}, and
  we demonstrate through extensive experiments that our method
  outperforms the state-of-the-art \texttt{Quick Boost} method.
\item We develop a constrained-oracle framework for deriving
  feature-example lower bounds on the problem that applies to a wide
  family of conceivable algorithms for the task, including our
  algorithm and \texttt{Quick Boost}, and we demonstrate that our
  algorithm is near-optimal within this family of algorithms through
  extensive experiments.
\item Within the constrained-oracle framework, we also derive a
  feature-example lower bound applicable to any algorithm solving the
  task, and we demonstrate that our algorithm empirically achieves
  performance close to this lower bound as well.
\end{itemize}

We will next expand on the ideas that underlie
our three main results above and discuss related work.

\paragraph{The Multi-Armed Bandit (MAB) Inspiration.}
Our approach to efficiently splitting decision tree nodes is based on
identifying intervals which contain the score (e.g. classifier's training accuracy) of each possible split and tightening those
intervals by observing training examples incrementally.
We can eventually exclude entire features
from further consideration because their intervals do not
overlap the intervals of the best splits.
Under this paradigm, the optimal strategy would be to assess all
examples for the best feature,
reducing its interval to an exact value,
and only then to assess examples for the remaining features
to rule them out.
Of course, we do not know in advance which feature is best.
Instead, we wish to spend our assessments optimally to identify the
best feature with the fewest assessments spent on the other features.
This corresponds well to the best arm identification problem studied
in the MAB literature. This insight inspired our training algorithm.

A ``Pure Exploration'' MAB algorithm in the ``Fixed-Confidence'' setting
\citep{DBLP:conf/icml/KalyanakrishnanTAS12,NIPS2012_4640,COLT13}
is given a set of arms (probability distributions over rewards)
and returns the arm with highest expected reward with high probability
(subsequently, WHP)
while minimizing the number of samples drawn from each arm.
Such confidence interval algorithms are generally categorized as
LUCB (Lower Upper Confidence Bounds) algorithms, because at each round
they ``prune'' sub-optimal arms whose confidence intervals do not overlap
with the most promising arm's interval until it is confident that WHP
it has found the best arm. 

In contrast to the MAB setting where one estimates the expected reward
of an arm WHP,
in the Boosting setting one can calculate the exact (training) accuracy
of a feature (expected reward of an arm) if one is willing to assess that
feature on all training examples.
When only a subset of examples are assessed, one can also calculate a
non-probabilistic ``uncertainty interval'' which is guaranteed to contain
the feature's true accuracy. This interval shrinks in proportion to the boosting weight of the assessed examples. We specialize the generic LUCB-style MAB algorithm of the best arm identification to assess examples in decreasing order of boosting weights, and to use uncertainty intervals in place of the more typical probabilistic confidence intervals.

\paragraph{Our Lower Bounds.}
We introduce two empirical lower bounds on the total number of examples needed
to be assessed in order to identify the exact greedy-optimal node for a given
set of boosting weights.
Our first lower bound is for the class of algorithms which assess feature
accuracy by testing the feature on examples in order of decreasing Boosting
weights (we call this the \emph{assessment complexity} of the problem).
We show empirically that our algorithm's performance is consistently nearly
identical to this lower bound.
Our second lower bound permits examples to be assessed in any order.
It requires a feature to be assessed with the minimal set of examples
necessary to prove that its training accuracy is not optimal.
This minimal set depends on the boosting weights in a given round,
from which the best possible (weighted) accuracy across all weak hypotheses
is calculated.
For non-optimal features, the minimal set is then identified using
Integer Linear Programming.

\subsection{Related Work}\label{related-work}

Much effort has gone to reducing the overall computational complexity of
training Boosting models.
In the spirit of \citet{icml2013_appel13}, which has the state-of-the-art
exact optimal-greedy boosted decision tree training algorithm \texttt{Quick Boost}
(our main competitor), we divide these attempts into three categories and
provide examples of the literature from each category: reducing
1) the set of features to focus on;
2) the set of examples to focus on; and/or
3) the training time of decision trees.
Note that these categories are independent of and parallel to each other.
For instance, 3), the focus of this work, can build a decision tree from
any subset of features or examples.
We show improvements compared to state-of-the-art algorithm both on subsets
of the training data and on the full training matrix.
Popular approximate algorithms such as
XGBoost \citep{Chen:2016:XST:2939672.2939785}
typically focus on 1) and 2)
and could benefit from using our algorithm for their training step.

Various works \citep{4270071, PaulBiswajitAthithanEtAl} focus on reducing the set of features.
\citet{busafekete:in2p3-00614564} divides features into subsets
and at each round of boosting uses adversarial bandit models to find the most promising subset for boosting. \texttt{LazyBoost} \citep{Escudero:2001:ULW:2387364.2387381} samples a subset of features uniformly at random to focus on at a given boosting round. 

Other attempts at computational complexity reduction involve sampling a set of
examples.
Given a fixed budget of examples, \texttt{Laminating}
\citep{Dubout:2014:ASL:2627435.2638580} attempts to find the best among a set of
hypotheses by testing each surviving hypothesis on a increasingly larger set of
sampled examples while pruning the worst performing half and doubling the number of examples, until it is left
with one hypthesis. It returns this hypothesis to boosting as the best one with probability $1-\delta$. The hypothesis identification part of \texttt{Laminating} is fairly identical to the best arm identification algorithm \texttt{Sequential Halving} \citep{icml2013_karnin13}. \texttt{Stochastic Gradient Boost} \citep{FriedmanStochasticGB}, and the weight trimming approach of \citet{Friedman98additivelogistic} are a few other intances of reducing the set of examples. \texttt{FilterBoost} \citep{NIPS2007_3321} uses an oracle to sample a set of examples from a very large dataset and uses this set to train a weak learner.

Another line of research focuses on reducing the training time of decision trees
\citep{implementing-decision-trees-and-forests-on-a-gpu, articleWuEtAl}.
More recently, \citet{icml2013_appel13} proposed \texttt{Quick Boost}, which
trains decision tree as weak learners while pruning underperforming
features earlier than a classic Boosting algorithm would.
They build their algorithm on the insight that the (weighted) error rate of
a feature when trained on a subset of examples can be used to bound its error
rate on all examples.
This is because the error rate is simply the normalized sum of the weights of the misclassified examples; if one supposes that all unseen examples may be
correctly classified, that yields a lower bound on the error rate.
If this lower bound is above the best observed error rate of a feature
trained on all examples, the underperforming feature may be pruned and no
more effort spent on it.

Our \texttt{Adaptive-Pruning Boost} algorithm carries forward the ideas
introduced by \texttt{Quick Boost}.
In contrast to \texttt{Quick Boost},
our algorithm is parameter-free and adaptive.
Our algorithm uses fewer training examples and thus faster training CPU time than
\texttt{Quick Boost}.
It works by gradually adding weight to the
``winning'' feature with the smallest upper bound on, e.g., its error rate
and the ``challenger'' feature with smallest lower bound,
until all challengers are pruned.
We demonstrate consistent improvement over \texttt{Quick Boost} on a
variety of datasets,
and show that when speed improvements are more modest this is due to
\texttt{Quick Boost} approaching the lower bound more tightly rather than
due to our algorithm using more examples than are necessary.
Our algorithm is consistently nearly-optimal in terms of the lower bound
for algorithms which assess examples in weight order,
and this lower bound in turn is close to the global lower bound.
Experimentally, we show that the reduction in total assessed examples
also reduces the CPU time.

\section{Setup and Notation}
We adopt the setup, 
description and notation of~\citet{icml2013_appel13} for ease of
comparison.  

\paragraph {A Generic Boosting Algorithm.}
Boosting algorithms train a linear combination of classifiers
$\mathcal{H}_T(x)=\sum^T_t {\alpha_t h_t(x)}$
such that an error function $\mathcal{E}$ is minimized by optimizing scalar
$\alpha_t$ and the weak learner $h_t(x)$ at round $t$.
Examples $x_i$ misclassified by $h_t(x)$ are assigned ``heavy'' weights $w_i$
so that the algorithm focuses on these heavy weight examples when training weak
learner $h_{t+1}(x)$ in round $t+1$.
Decision trees, defined formally below, are often used as weak learners.

\paragraph {Decision Tree.}  A binary decision tree $h_{\textit{Tree}}(x)$ is a
tree-based classifier where every non-leaf node is a decision stump
$h(x)$.  A decision stump can be viewed as a tuple $(p, k, \tau)$ of a
polarity (either $+1$ or $-1$), the feature column index, and
threshold, respectively, which predicts a binary label from the set
$\{+1, -1\}$ for any input $x \in \mathbb{R}^K$ using the function
$h(x) \equiv p\mathop{\mathrm{sign}}(x[k] - \tau)$.

A decision tree $h_{\textit{Tree}}(x)$ is trained, top to bottom, by
``splitting'' a node, i.e. selecting a stump $h(x)$ that optimizes
some function such as error rate, information gain, or GINI impurity.  While this paper focuses on selecting stumps based on
error rate, we provide bounds for information gain in the supplementary material which can be used to split nodes on information gain.
Our algorithm \texttt{Adaptive-Pruning Stump} (Algorithm~\ref{adaptive-pruning-stump}), a subroutine of \texttt{Adaptive-Pruning Boost} (Algorithm~\ref{boosting}), trains a decision stump $h(x)$ with
fewer total example assessments than its analog, the subroutine of the-state-of-the-art algorithm
\texttt{Quick Boost}, does. Note that \texttt{Adaptive-Pruning Stump} used iteratively can train a decision tree, but for simplicity we assume our weak learners are binary decision stumps.
While we describe \texttt{Adaptive-Pruning Stump} for
binary classification, the reasoning also applies to multi-class data.

To describe how \texttt{Adaptive-Pruning Stump} trains a stump
we need a few definitions.
Let $n$ be the total number of examples, and $m \le n$ some number of examples
on which a stump has been trained so far.
We will assume that Boosting provides the examples in decreasing weight
order.
This order can be maintained in $O(n)$ time in the presence of Boosting weight updates
because examples which are correctly classified do not change their relative
weight order, and examples which are incorrectly classified do not change their
relative weight order; a simple merge of these two groups suffices.
We can therefore number our examples from 1 to $n$ in decreasing weight order.
Furthermore,
\begin{itemize}[topsep=0pt,itemsep=0pt]
\item let $Z_m := \sum_{i=1}^m {w_i}$ be sum of the weights of first $m$ (heaviest) examples, and
\item let $\epsilon_m := \sum_{i=1}^m {w_i \mathbbm{1} \{h(x_i) \ne y_i\}}$ be the sum of the weights of the examples from the first $m$ which are misclassified by the stump $h(x)$.
\end{itemize}

The weighted error rate for stump $j$ on the first $m$ examples is then
$E^j_m := \epsilon_m^j / Z_m$.

\section{Algorithm}\label{alg}
\texttt{Adaptive-Pruning Stump} prunes features based on
exact intervals (which we call uncertainty intervals) and returns the best feature
deterministically.
To do this we need lower bounds and upper bounds on the stump's
training error rate. 
Our lower bound assumes that all unseen examples are classified
correctly and our upper bound assumes that all unseen examples are classified
incorrectly.
We define $L_m^j$ as the lower bound on the error rate for
stump $j$ on all $n$ examples, when computed on the first $m$ examples,
and $U_m^j$ as the corresponding upper bound.
For any $1 \le m \le n$, we define,
using $c^j_i := \mathbbm{1} \{h_j(x_i) \ne y_i\}$ to indicate
whether stump $j$ incorrectly classifies example $i$,
\begin{align*}
	L_m^j :=
	 \frac{1}{Z_n}
		\sum_{i=1}^m {w_i c^j_i}
	\le  
	\underbrace{
	\frac{1}{Z_n} 
		\sum_{i=1}^n {w_i c^j_i}
	}_{E_n^j}
	\le  \frac{1}{Z_n}\left(
		\epsilon^j_m
		+ \sum_{i=m+1}^n w_i
		\right)
	=
	\frac{1}{Z_n}\left(
		\epsilon^j_m
		+ (Z_n - Z_m)
		\right)
		=: U_m^j.
\end{align*}
For any two stumps $i$ and $j$ when numbers $m$ and $m'$ exist
such that $L_m^i > U_{m'}^j$ then we can safely discard stump $i$, as it cannot
have the lowest error rate.
This extension of the pruning rule used by \citet{icml2013_appel13}
permits each feature to have its own interval of possible error
rates, and permits us to compare features for pruning without first
needing to assess all $n$ examples for any feature
(\texttt{Quick Boost}'s subroutine requires the
current-best feature to be tested on all $n$ examples).

Now we describe our algorithm in detail; see the
listing in Algorithm~\ref{adaptive-pruning-stump}.
We use $f_k$ to denote an object which stores all decision stumps $h(x)$ for feature
$x[k]$.
Recall that $x \in \mathbb{R}^K$ and that $x[k]$ is the $k_{th}$ feature of $x$,
for $k \in \{ 1, \dots, K \}$.
$f_k$ has method $assess(batch)$, when given a ``batch'' of examples, updates
$L_m$, $E_m$,  $U_m$ (defined above) for all decision stumps of feature $x[k]$ based
on the examples in the batch.
It also has methods $LB()$ and $UB()$, which report the $L_m$ and $U_m$ for the
single hypothesis with smallest error $E_m$ on the $m$ examples seen so far,
and $bestStump()$, which returns the hypothesis with smallest error $E_m$.

\texttt{Adaptive-Pruning Stump} proceeds until there is some feature $k^*$ whose upper bound
is below the lower bounds for all other features.
We then know that the best hypothesis uses feature $k^*$.
We assess any remaining unseen examples
for feature $k^*$ in order to identify the best threshold and polarity
and to calculate $E^{k^*}_n$.
Thus, our algorithm always finds the exact greedy-optimal hypothesis.

In order to efficiently compare two features $i$ and $j$ to decide whether
to prune feature $i$,
we want to ``add'' the minimum weight to these arms to
possibly obtain that $L_m^i > U_{m'}^j$.
The most efficient way to do this is to test each feature against
a batch of the heaviest unseen examples whose weight is at least
the gap $U_{m'}^j - L_m^i$.
This permits us to choose batch sizes adaptively, based on the minimum weight
needed to prune a feature given the current boosting weights and the
current uncertainty intervals for each arm.
We note that our ``weight order'' lower bound on the sample complexity of the
problem in the next section is also calculated based on this insight.
This is in contrast to \texttt{Quick Boost}, which accepts parameters to
specify the total number of batches and the weight to use for initial estimates;
the remaining weight is divided evenly among the batches.
When the number of batches chosen is too large, the run time of a training
round approaches $O(n^2)$; when it is too small, the run time approaches
that of assessing all $n$ examples.

\begin{wrapfigure}{R}{0.5\textwidth}
\vspace{-2.5em}
\begin{minipage}[t]{0.5\textwidth}
\begin{algorithm}[H]
\caption{\texttt{Adaptive-Pruning Stump}}\label{adaptive-pruning-stump}
\begin{algorithmic}
   \STATE {\bfseries Input:} Examples \{$x_1, \dots, x_n$\}, Labels \{$y_1, \dots, y_n$\}, Weights $\{w_{1},\dots,w_{n}\}$
    \STATE {\bfseries Output:} $h(x)$  
    \STATE $m \gets$ min. index s.t. $Z_m \ge 0.5$
    \FOR {$k = 1$ {\bfseries to} $K$}
        \STATE $f_k.assess([x_1, \dots, x_m]); m_k \gets m$
    \ENDFOR
    \STATE $a \gets k$ with min $f_k.UB()$
    \STATE $b \gets k \ne a$ with min $f_k.LB()$
    \WHILE {$f_a.UB() > f_b.LB()$}
        \STATE $gap \gets f_a.UB() - f_b.LB()$
        \STATE $m \gets$ min index s.t. $Z_m \ge Z_{m_a} + gap$
        \STATE $f_a.assess([x_{m_a+1}, \dots, x_m]); m_a \gets m$
        \STATE $gap \gets f_a.UB() - f_b.LB()$
        \IF {$gap > 0$}
            \STATE $m \gets$ min index s.t. $Z_m \ge Z_{m_b} + gap$
            \STATE $f_b.assess([x_{m_b+1}, \dots, x_m]); m_b \gets m$
        \ENDIF
        \IF {$f_a.UB() < f_b.UB()$}
            \STATE $a \gets b$
        \ENDIF
        \STATE $b \gets k \ne a$ with min $f_k.LB()$
    \ENDWHILE
    \STATE {\bfseries return} $h(x) := f_a.bestStump()$
\end{algorithmic}
\end{algorithm}
\end{minipage}
\begin{minipage}[t]{0.5\textwidth}
\begin{algorithm}[H]
\caption{\texttt{Adaptive-Pruning Boost}}\label{boosting}
\begin{algorithmic}
   \STATE {\bfseries Input:} Instances \{$x_1, \dots, x_n$\}, Labels \{$y_1, \dots, y_n$\}
   \STATE {\bfseries Output:} $\mathcal{H}_T(x)$
   \STATE {\bfseries Initialize Weights:} $\{w_{1},\dots,w_{n}\}$   \FOR{$t=1$ {\bfseries to} $T$}
   \STATE Train Decision Tree $h_{Tree}(x)$ one node at a time by calling \texttt{Adaptive-Pruning Stump}
   \STATE Choose $\alpha_t$ and update $\mathcal{H}_t(x)$
   \STATE Update and Sort (in descending order) $w$
   \ENDFOR
\end{algorithmic}
\end{algorithm}
\end{minipage}
\vspace{-6em}
\end{wrapfigure}

At each round, \texttt{Adaptive-Pruning Boost} trains a decision tree
in Algorithm~\ref{boosting} by calling the subroutine
\texttt{Adaptive-Pruning Stump} of Algorithm~\ref{adaptive-pruning-stump}.

\paragraph{Implementation Details.}
The $f_k.assess()$ implementation is shared across all algorithms.
 For $b$ batches of exactly $m$ examples each on a feature $k$ with $v$ distinct
 values, our implementation of $f_k.assess$ takes $O(b m \log (m + v))$
 operations.
 We maintain an ordered list of intervals of thresholds for each feature with
 the feature values for the examples assessed so far lying on the interval
 boundaries.
 Any threshold in the interval will thus have the same performance on all
 examples assessed so far.
 To assess a batch of examples, we sort the examples in the batch by feature
 value and then split intervals as needed and calculate scores for the
 thresholds on each interval in time linear in the batch size and number of
 intervals.

Note also that maintaining the variables $a$ and $b$ requires a
single heap, and that in many iterations of the \texttt{while} loop we can
update these variables from the heap in constant time
(e.g. when $b$ has not changed, when $a$ and $b$ are simply swapped,
or when $b$ can be pruned).

\section{Lower Bounds}\label{lb}

\begin{wrapfigure}{R}{0.5\textwidth}
\centering
\begin{subfigure}{.5\textwidth}
\includegraphics[width=\linewidth]{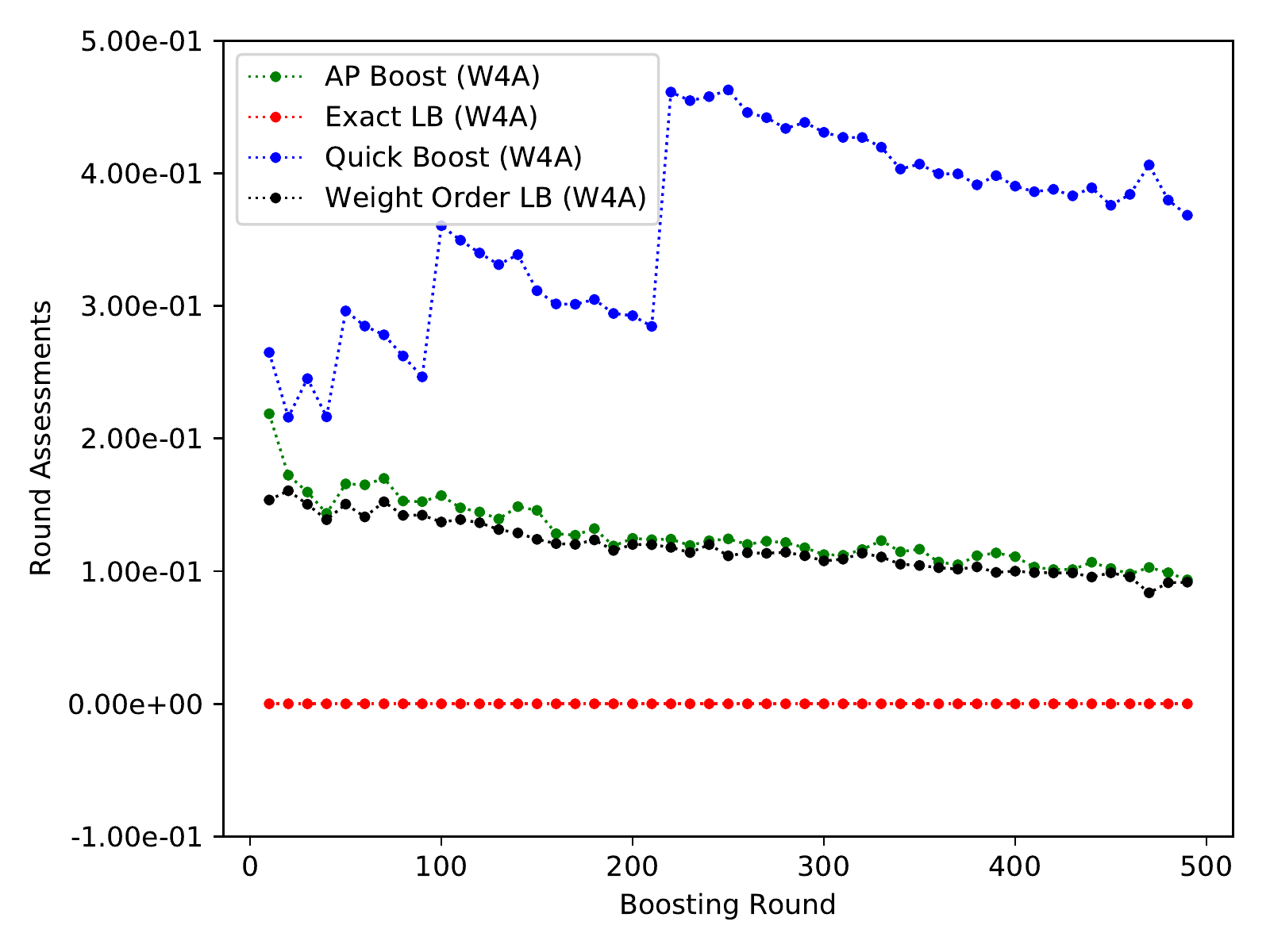}
\end{subfigure}
\begin{subfigure}{.5\textwidth}
\includegraphics[width=\linewidth]{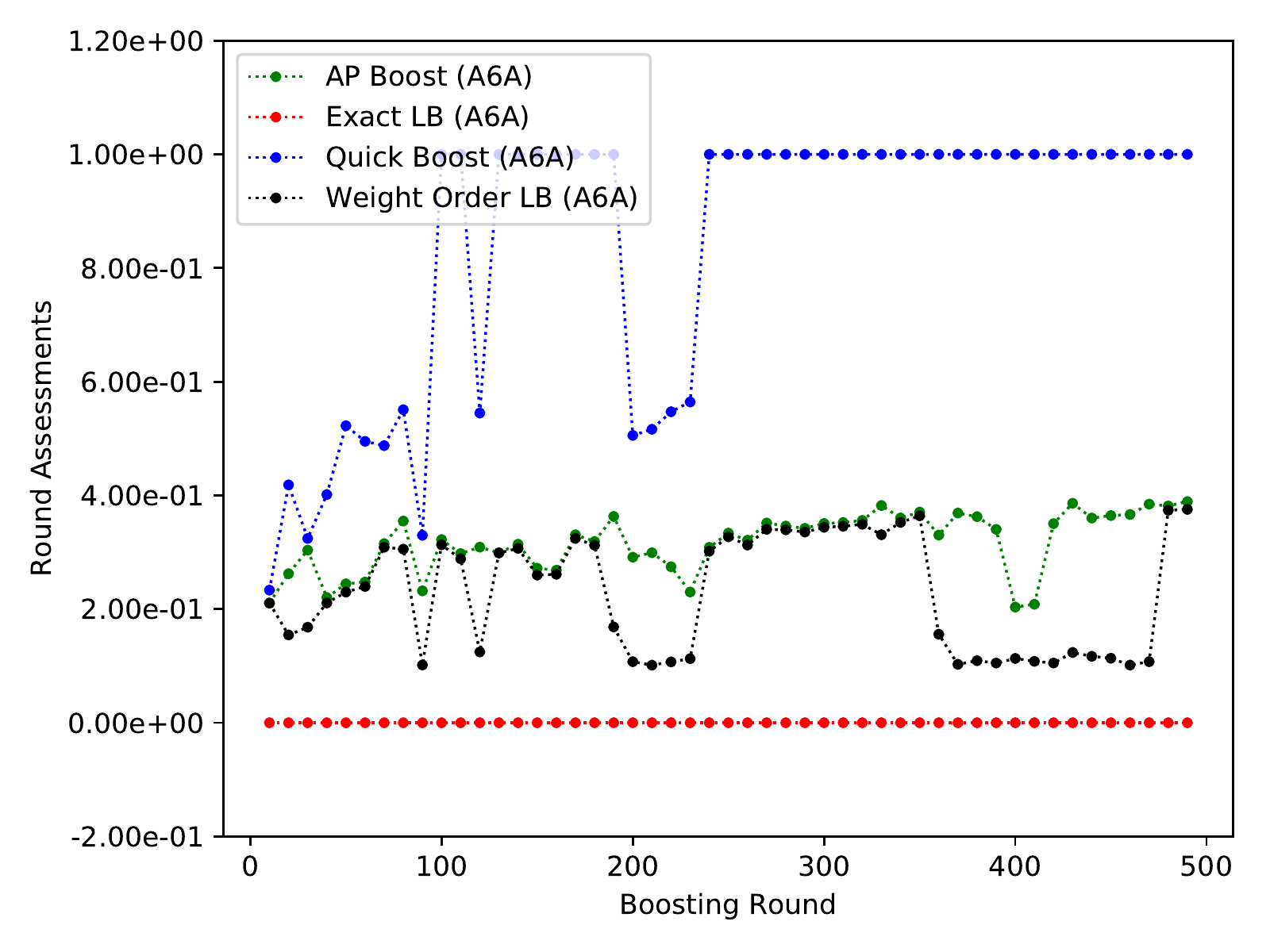}
\end{subfigure}
\caption{Lower Bounds versus Upper Bounds. Datasets W4A (top) and A6A (bottom) were used with trees of depth 1.
The y-axis is the \emph{fraction of the gap}
between the exact lower bound
(at zero) and the full corpus size (at one) which an algorithm used in a
given round.
Non-cumulative example assessments are plotted for every 10 rounds.
}
\label{fig:elb}
\vspace{-3em}
\end{wrapfigure}

We compare \texttt{Adaptive-Pruning Boost} against two lower bounds, defined empirically based
on the boosting weights in a given round.
In our \emph{weight order lower bound}, we consider the minimum number
of examples required to determine that a given feature is underperforming
with the assumption that examples will be assessed in order of decreasing
boosting weight.
Our \emph{exact lower bound} permits examples to be assessed in any order,
and so bounds any possible algorithm which finds the best-performing feature.

\paragraph{Weight Order Lower Bound.}
For this bound, we first require that \texttt{Adaptive-Pruning Stump} selects the feature with
minimal error.
In the case of ties, an optimal feature may be chosen arbitrarily.
\texttt{Adaptive-Pruning Stump} need to assess every example for the returned
feature in order for \texttt{Adaptive-Pruning Boost} to calculate $\alpha$ and update weights $w$ , so the lower
bound for the returned feature is simply the total number of examples $n$.

Let $k^*$ be the returned feature, and $E^*$ its error rate
when assessed on all $n$ examples.
For any feature $k \ne k^*$ which is not returned, we need to prove that it is
underperforming (or tied with the best feature).
Let $J_k$ be the set of decision stumps which use feature $k$;
then we need to find the smallest value $m$ such that for all stumps
$j \in J_k$, we have $L^j_m \ge E^*$.
Our lower bound is simply 
$	LB_{wo} :=
	n + \sum_{k \ne k^*} \min \Set{m : \forall j \in J_k, L^j_{m} \ge E^*}$.
We present results in Figure~\ref{fig:wolb} showing that \texttt{Adaptive-Pruning Boost}
achieves this bound on a variety of datasets.
\texttt{Quick Boosting}, in contrast, sometimes approaches
this bound but often uses more examples than necessary.

\paragraph{Exact Lower Bound.}
In order to test the idea that adding examples in weight order is nearly optimal,
and to provide a lower bound on \emph{any} algorithm which finds the
optimal stump, we also present an exact lower bound on the problem.
Like the weight order lower bound, this bound is defined in terms of the boosting
weights in a given round; unlike it, examples may be assessed in any order.
It is not clear how one might achieve the exact lower bound without incurring an
additional cost in time.
We leave such a solution to future work.
However, we show in Figure~\ref{fig:elb} that this bound is, in fact,
very close to the weight order lower bound.

For the exact lower bound, we still require the selected feature $k^*$ to be
assessed against all examples; this is imposed by the boosting algorithm.
For any other feature $k \ne k^*$, we simply need the size of the smallest
set of examples which would prune the feature (or prove it is tied with $k^*$).
We will use $M \subseteq \Set{1,\dots,n}$ to denote a set of indexes of examples
assessed for a given feature,
and $L^j_M$ to denote the lower bound of stump $j$ when assessed on the
examples in subset $M$.
This bound, then, is
$	LB_{exact} :=
	n + \sum_{k \ne k^*} \min_{M : L^j_M \ge E^*} |M|$.

We identify the examples included in the smallest subset $M$
for a given feature $k \ne k^*$ using integer linear programming.
We define binary variables $c_1,\dots,c_n$, where $c_i$ indicates whether example
$i$ is included in the set $M$.
We then create a constraint for each stump $j \in J_k$ defined for feature $k$
which requires that the stump be proven underperforming.
Our program, then, is:
$
	\texttt{Minimize } 
	 \sum_{i=1}^n c_i
	\texttt{ s.t. }
	c_i \in \Set{0,1}
		~~~ \forall i 
, \texttt{ and }
	\sum_{i=1}^n c_i w_i \1{ h_j(x_i) \ne y_i }
		\ge E^*
		~~~ \forall j \in J_k$.

\paragraph{Discussion.}
Figure~\ref{fig:elb} shows a non-cumulative comparison of our weight order lower bound to the global lower bound. Minimizing the global lower bound function mentioned above is computationally expensive. For this reason we used binary class datasets of moderate size and trees of depth 1 as weak leaners, but we have no reason to believe that the technique would not work for deeper trees and multi-class datasets. Refer to Table~\ref{datasets} for details of datasets.
The weight order lower bound and \texttt{Adaptive-Pruning Boost} 
are within 10-20\% of the exact lower bound,
but \texttt{Quick Boost} often uses half to all of the unnecessary
training examples in a given round.

\section{Experiments}
We experimented with shallow trees on various binary and multi-class datasets.
We report both assessment complexity and CPU time complexity for each dataset. Though \texttt{Adaptive-Pruning Boost} is a general Boosting algorithm, we experimented with the following class of algorithms (1) Boosting exact greedy-optimal decision trees and (2) Boosting approximate decision trees.

Each algorithm was run with either the state-of-the-art method (\texttt{Quick Boost}) or our decision tree training method (\texttt{Adaptive-Pruning Boost}), apart from the case of Figure~\ref{fig:wolb} that also uses the brute-force decision tree search method (\texttt{Classic AdaBoost}). The details of our datasets are in Table~\ref{datasets}. For datasets SATIMAGE, W4A, A6A, and RCV1 tree depth of three was used and for MNIST Digits tree depth of four was used (as in \citet{icml2013_appel13}). Train and test error results are provided as supplementary material.

\begin{table*}[ht]
\caption{The datasets used in our experiments.}
\label{datasets}
\vskip 0.15in
\begin{center}
\begin{small}
\begin{sc}
\begin{tabular}{llccc}
\toprule
Dataset & Source & Train / Test Size & Total Features & Classes\\
\midrule
a6a & \citet{Platt:1999:FTS:299094.299105} & 11220 / 21341 & 123 & 2 \\
MNIST Digits & \citet{726791} & 60000 / 10000 & 780 & 10 \\
rcv1 (Binary) & \citet{Lewis:2004:RNB:1005332.1005345} & 20242 / 677399 & 47236 & 2 \\
satimage & \citet{991427} & 4435 / 2000 & 36 & 6 \\
w4a & \citet{Platt:1999:FTS:299094.299105} & 7366 / 42383 & 300 & 2 \\
\bottomrule
\end{tabular}
\end{sc}
\end{small}
\end{center}
\vskip -0.1in
\end{table*}

\paragraph{Boosting Exact Greedy-Optimal Decision Trees.}
We used \texttt{AdaBoost} for exact decision tree training.
Figure~\ref{fig:wolb} shows the total number of example assessments used by AdaBoost when it uses three different decision trees building methods described above. In all of these experiments, our algorithm, \texttt{Adaptive-Pruning Boost}, not only consistently beats \texttt{Quick Boost} but it also almost matches the weight order lower bound. The \texttt{Classic AdaBoost} can be seen as the upper bound on the total number of example assessments.

Table~\ref{complexity-exact-results} shows that CPU time improvements correspond to example-assessments improvements for \texttt{Adaptive-Pruning Boost} for all our datasets, except for RCV1. This could be explained by Figure~\ref{fig:wolb} wherein \texttt{Quick Boost} is seen approaching the lower bound for this particular dataset. While \texttt{Adaptive-Pruning Boost} is closer to the lower bound, its example-assessments improvements are not enough to translate to CPU time improvements.

 \begin{figure*}[ht]
\centering
\begin{subfigure}{.31\linewidth}
\includegraphics[width=\linewidth]{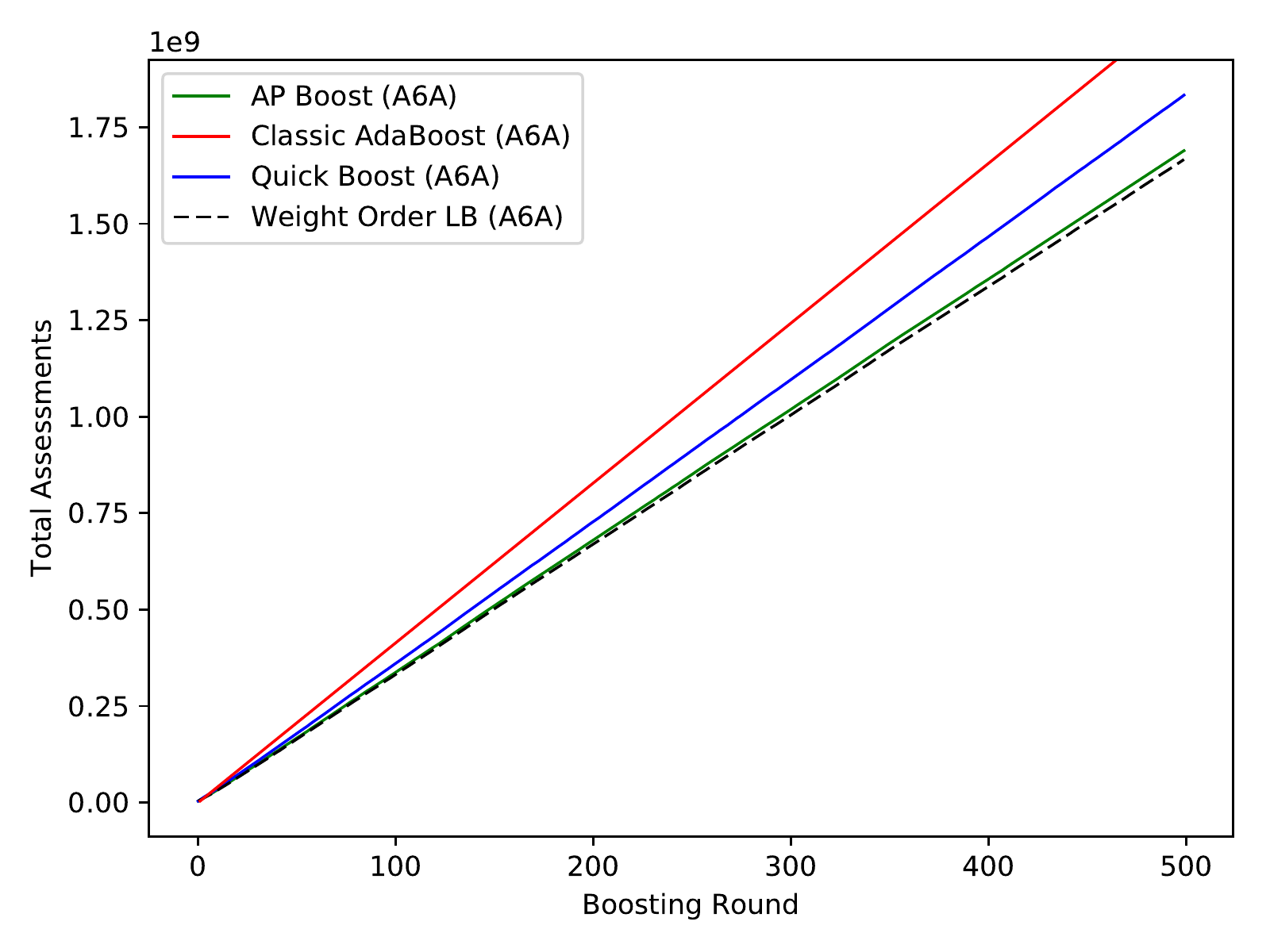}
\end{subfigure}
\begin{subfigure}{.31\linewidth}
\includegraphics[width=\linewidth]{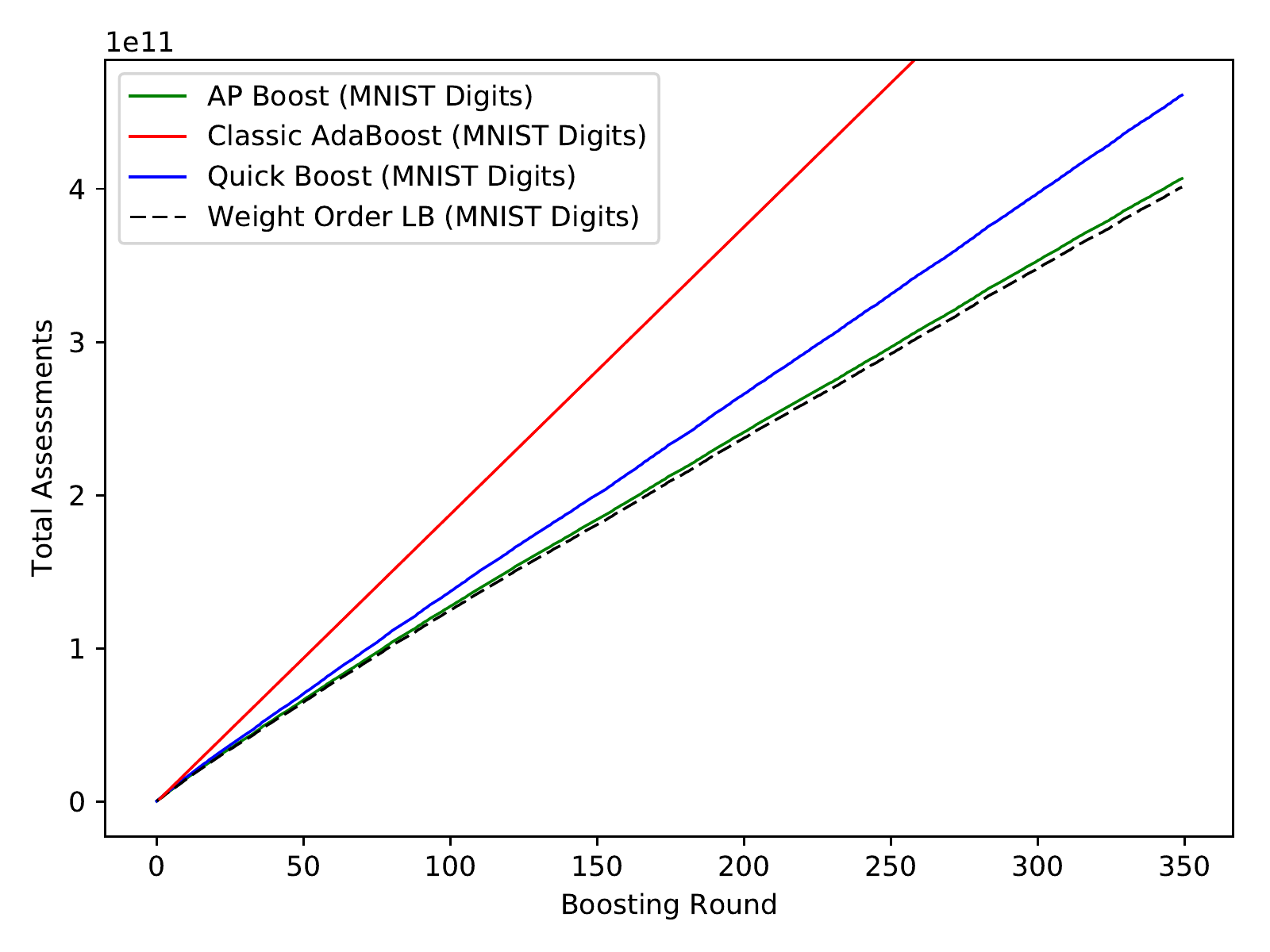}
\end{subfigure}
\begin{subfigure}{.31\linewidth}
\includegraphics[width=\linewidth]{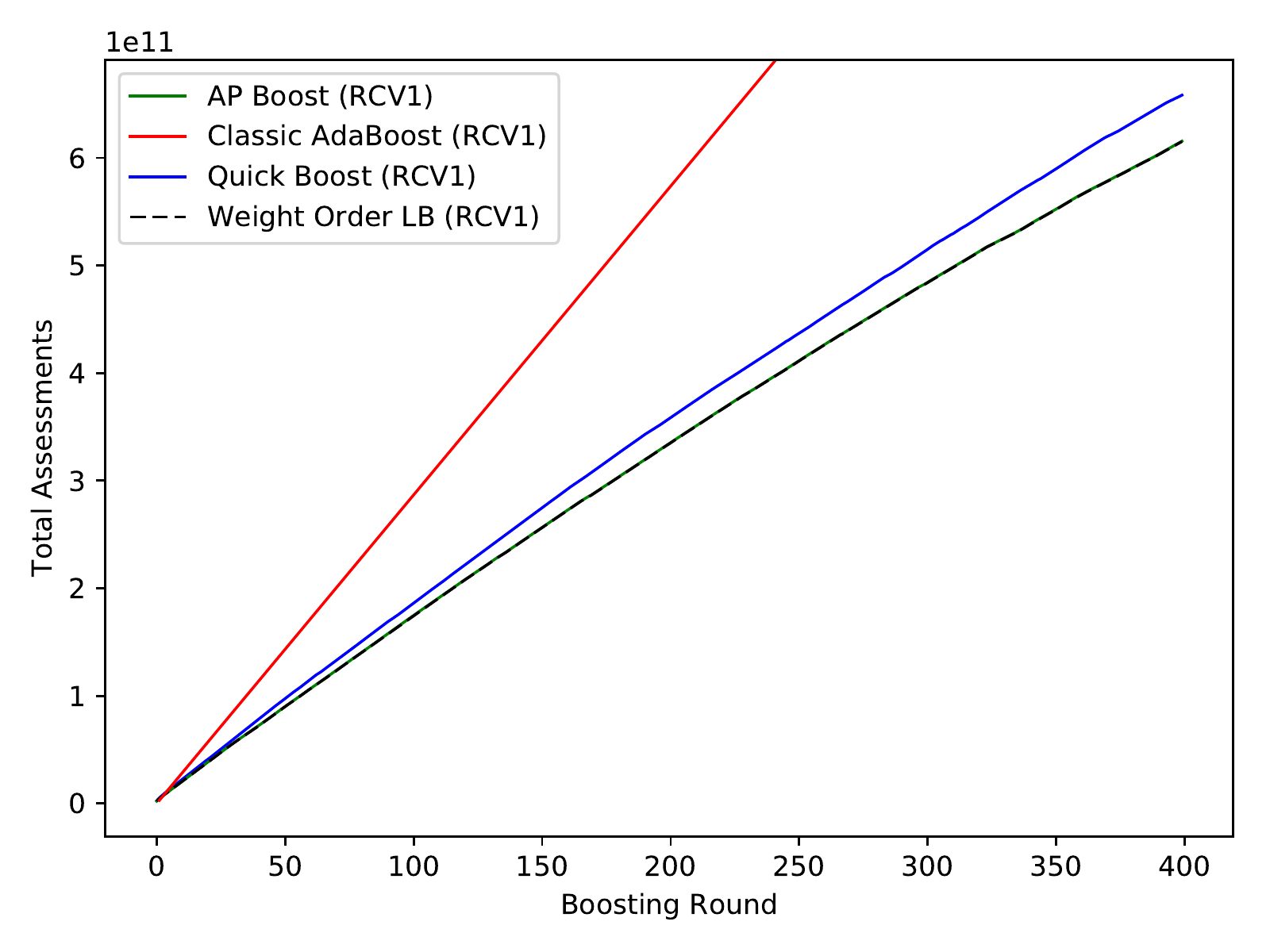}
\end{subfigure}
\begin{subfigure}{.31\linewidth}
\includegraphics[width=\linewidth]{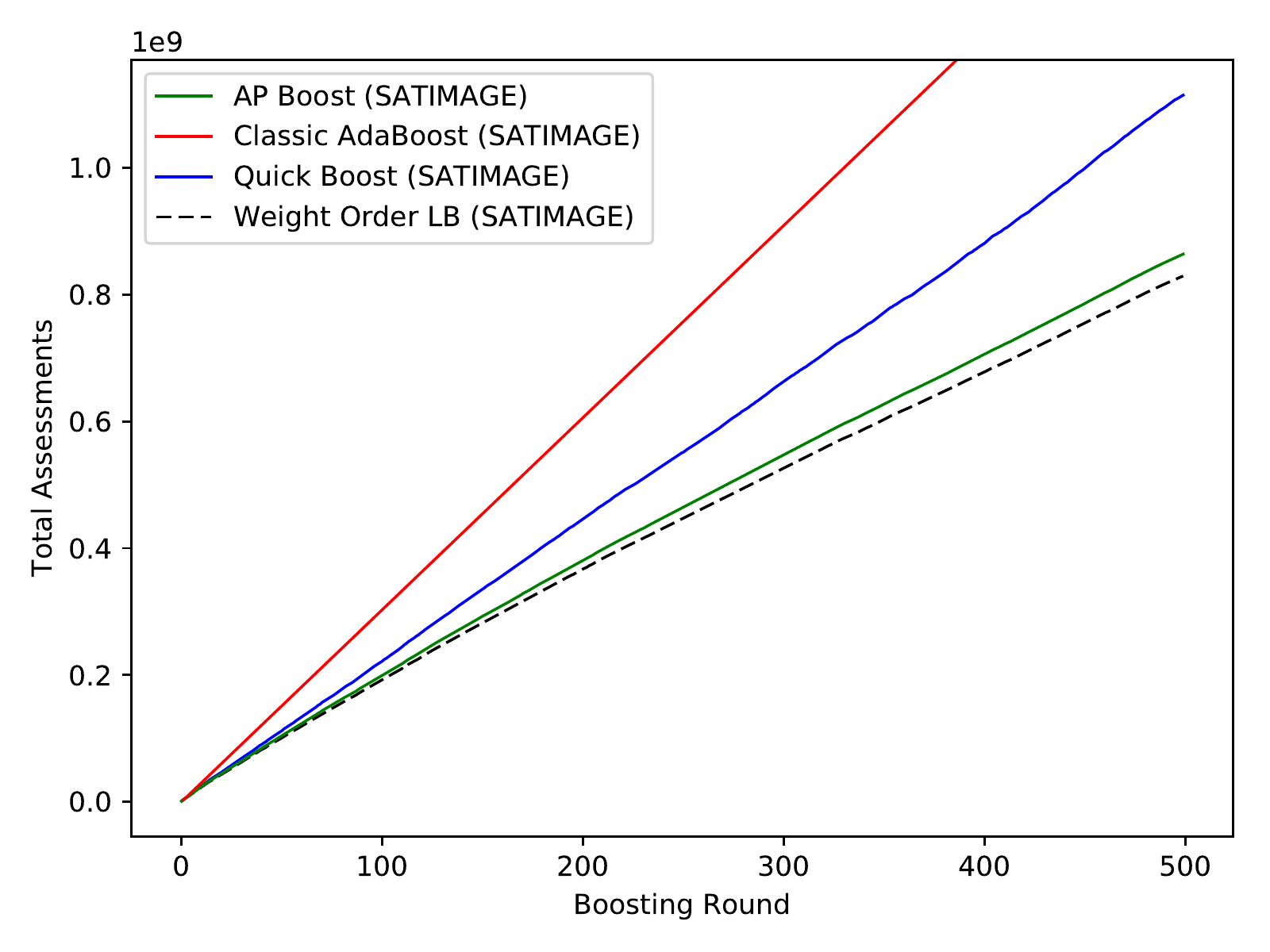}
\end{subfigure}
\begin{subfigure}{.31\linewidth}
\includegraphics[width=\linewidth]{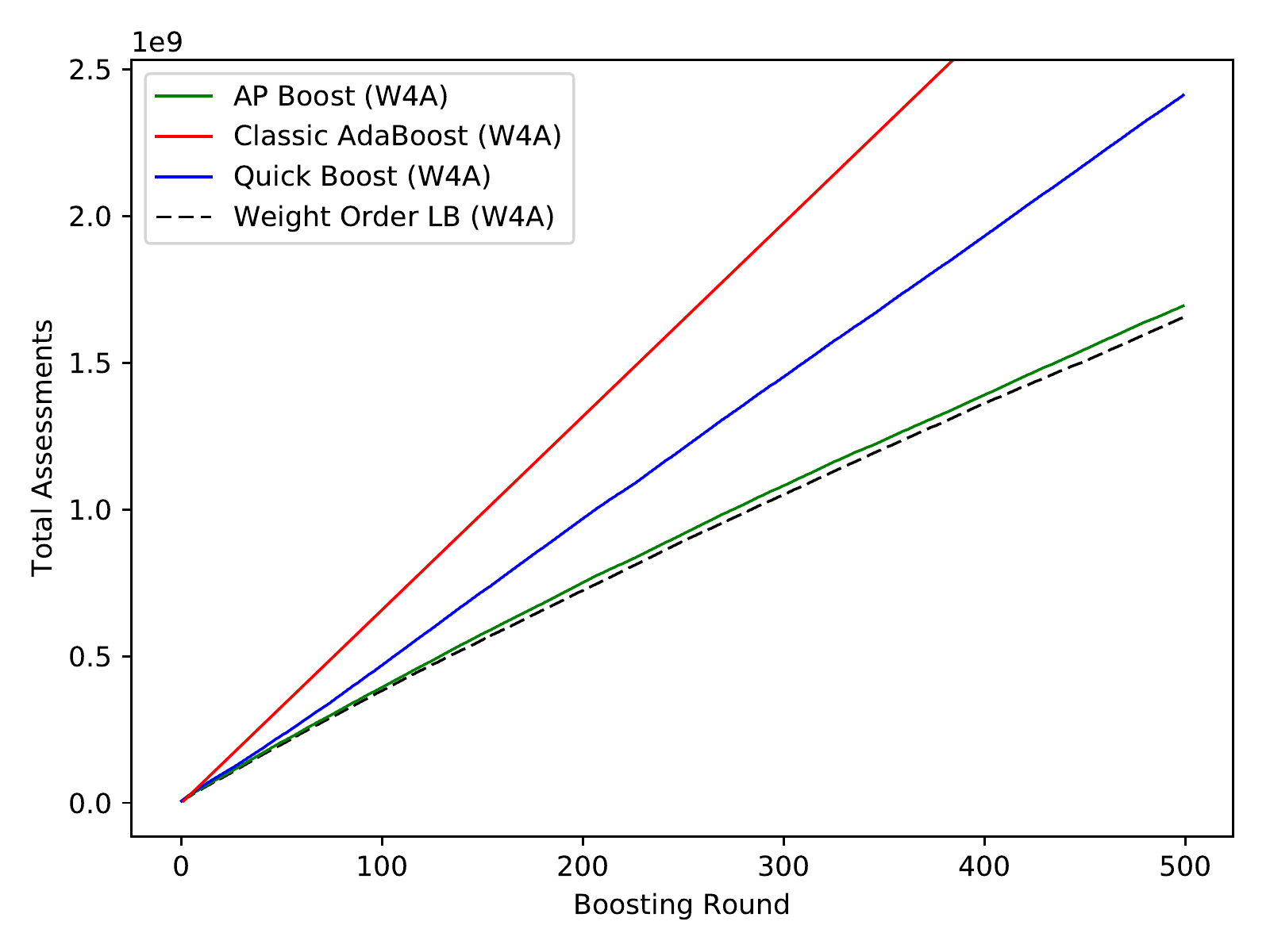}
\end{subfigure}
\caption{We report the total number of assessments at various boosting rounds used by the algorithms, as well as the weight order lower bound. In all of these experiments, our algorithm, \texttt{AP Boost}, not only consistently beats \texttt{Quick Boost} but it also almost matches the lower bound.}
\label{fig:wolb}
\end{figure*}

\begin{table*}[ht]
\caption{Computational Complexity for AdaBoost.
All results are for 500 rounds of boosting except MNIST (300 rounds) and RCV1 (400 rounds).}
\label{complexity-exact-results}	
\vskip 0.15in
\begin{center}
\begin{small}
\begin{sc}
\begin{tabular}{llccrccr}
\toprule
& & \multicolumn{3}{c}{CPU Time in Seconds} & \multicolumn{3}{c}{\# Example Assessments} \\
Dataset & Boosting & AP-B & QB & Improv. & AP-B & QB & Improv. \\
\midrule
a6a & AdaBoost & 4.49e+02 & 4.46e+02 & 5.3\% & 1.69e+09 & 1.83e+09 & 7.8\% \\
mnist & AdaBoost & 6.32e+05 & 6.60e+05 & 4.2\% & 3.52e+11 & 3.96e+11 & 11.1\% \\
rcv1 & AdaBoost & 1.58e+05 & 1.58e+05 & -0.5\% & 6.15e+11 & 6.58e+11 & 6.5\% \\
satimage & AdaBoost & 9.21e+02 & 1.19e+03 & 18.9\% & 8.64e+08 & 1.11e+09 & 22.5\% \\
w4a & AdaBoost & 3.03e+02 & 3.96e+02 & 27.1\% & 1.69e+09 & 2.41e+09 & 29.8\% \\
\midrule
Mean &  & &  & 11\% & & & 15.54\% \\
\bottomrule
\end{tabular}
\end{sc}
\end{small}
\end{center}
\vskip -0.1in
\end{table*}

\paragraph{Boosting Approximate Decision Trees.}
We used two approximate boosting algorithms.
We experimented with Boosting with Weight-Trimming 90\% and 99\% \citep{Friedman98additivelogistic}, wherein the weak hypothesis is trained only on 90\% or 99\% of the weights, and LazyBoost 90\% and 50\% \citep{Escudero:2001:ULW:2387364.2387381} wherein the weak hypothesis is trained only on 90\% or 50\% randomly selected features. Table~\ref{complexity-approx-results} shows that the CPU time improvements correspond to assessment improvements.

Note that approximate algorithms like XGBoost of \cite{Chen:2016:XST:2939672.2939785} are not competitors to \texttt{Adaptive-Pruning Boost} but rather potential ``clients'' because such algorithms train on a subset of the data. Therefore, they are not appropriate baselines to our method.
\begin{table*}[ht]
\caption{Computational Complexity for LazyBoost and Boosting with Weight Trimming.
All results are for 500 rounds of boosting except MNIST (300 rounds) and RCV1 (400 rounds).}
\label{complexity-approx-results}	
\vskip 0.15in
\begin{center}
\begin{small}
\begin{sc}
\begin{tabular}{llccrccr}
\toprule
& & \multicolumn{3}{c}{CPU Time in Seconds} & \multicolumn{3}{c}{\# Example Assessments} \\
Dataset & Boosting & AP-B & QB & Improv. & AP-B & QB & Improv. \\
\midrule
a6a & LazyBoost (0.5) & 1.86e+02 & 1.95e+02 & 4.8\% & 8.48e+08 & 9.22e+08 & 8.1\%  \\
mnist & LazyBoost (0.5)  & 3.46e+05 & 3.52e+05 & 1.8\% & 1.87e+11 & 2.07e+11 & 9.7\% \\
rcv1 & LazyBoost (0.5) & 7.86e+04 & 7.54e+04 & -4.2\% & 3.18e+11 & 3.29e+11 & 3.4\% \\
satimage & LazyBoost (0.5) & 4.70e+02 & 5.48e+02 & 14.2\%  & 5.17e+08 & 6.11e+08 & 15.4\%  \\
w4a & LazyBoost (0.5) & 1.15e+02 & 1.58e+02 & 26.8\% & 8.61e+08 & 1.22e+09 & 29.3\% \\
\midrule
Mean &  & &  & 8.68\% & & & 13.18\% \\
\midrule
a6a & LazyBoost (0.9) & 3.28e+02 & 3.48e+02 & 5.6\%  & 1.51e+09 & 1.64e+09 & 7.7\%  \\
mnist & LazyBoost (0.9) & 5.89e+05 & 6.09e+05 & 3.3\% & 3.20e+11 & 3.59e+11 & 10.9\% \\
rcv1 & LazyBoost (0.9) & 1.38e+05 & 1.37e+05 & -1.0\% & 5.60e+11 & 5.93e+11 & 5.6\% \\
satimage & LazyBoost (0.9) & 7.37e+02 & 8.89e+02 & 17.1\%  & 8.05e+08 & 1.01e+09 & 20\%  \\
w4a & LazyBoost (0.9) & 2.04e+02 & 2.82e+02 & 27.7\% & 1.52e+09 & 2.19e+09 & 30.5\% \\
\midrule
Mean &  & &  & 10.54\% & & & 14.94\% \\
\midrule
a6a & Wt. Trim (0.9) & 2.69e+02 & 2.69e+02 & 0\%  & 1.23e+09 & 1.24e+09 & 1.4\%  \\
mnist & Wt. Trim (0.9) & 6.42e+05 & 8.02e+05 & 19.9\% & 4.61e+11 & 4.61e+11 & 0\% \\
rcv1 & Wt. Trim (0.9) & 8.87e+04 & 8.95e+04 & 0.9\%  & 3.65e+11 & 3.79e+11 & 3.6\% \\
satimage & Wt. Trim (0.9) & 9.87e+02 & 9.76e+02 & -1.2\%  & 1.26e+09 & 1.26e+09 & 0.1\%  \\
w4a & Wt. Trim (0.9) & 1.88e+02 & 1.96e+02 & 4.1\% & 1.40e+09 & 1.43e+09 & 2.5\% \\
\midrule
Mean &  & &  & 4.74\% & & & 1.52\% \\
\midrule
a6a & Wt. Trim (0.99) & 3.34e+02 & 3.38e+02 & 1.3\%  & 1.54e+09 & 1.58e+09 & 2.6\%  \\
mnist & Wt. Trim (0.99) &  5.80e+05 & 5.69e+05 & -1.8\% & 3.16e+11 & 3.37e+11 & 6.1\% \\
rcv1 & Wt. Trim (0.99) & 1.38e+05 & 1.37e+05 & -1.0\% & 5.61e+11 & 5.86e+11 & 4.4\% \\
satimage & Wt. Trim (0.99) & 6.49e+02 & 6.68e+02 & 2.9\%  & 7.01e+08 & 7.39e+08 & 5.1\%  \\
w4a & Wt. Trim (0.99) & 1.91e+02 & 2.03e+02 & 6.0\% & 1.44e+09 & 1.52e+09 & 5.3\% \\
\midrule
Mean &  & &  & 1.48\% & & & 4.7\% \\
\bottomrule
\end{tabular}
\end{sc}
\end{small}
\end{center}
\vskip -0.1in
\end{table*}

\section{Conclusion}
In this paper, we introduced an efficient exact greedy-optimal algorithm, \texttt{Adaptive-Pruning Boost}, for boosted decision trees. Our experiments on various datasets show that our algorithm use fewer total example assessments compared to the-state-of-the-art algorithm \texttt{Quick Boost}. We further showed that \texttt{Adaptive-Pruning Boost} almost matches the lower bound for its class of algorithms and the global lower bound for any algorithm.

\clearpage

\bibliography{paper, biblioEmilie}
\bibliographystyle{plainnat}

\clearpage

\begin{appendices}
\section{Additional Results}
\subsection{Train and Test Error for AdaBoost}
Table~\ref{tbl:AdaBoostTestTrainErrors} reports test and train errors at various Boosting rounds. Our algorithm achieves the test and train error in fewer total number of example assessments, compared to \texttt{Quick Boost}. Note that both algorithms, except in the case of RCV1, have the same test and train error at a given round, as they should because both train identical decision trees. The case of RCV1 is due to the algorithms picking a weak learner arbitrarily in case of ties, without changing the overall results significantly.
\begin{table*}[ht]
\caption{AdaBoost results,
reported at rounds 100, 300 and 500 (400 for RCV1).}
\label{tbl:AdaBoostTestTrainErrors}
\vskip 0.15in
\begin{center}
\begin{small}
\begin{sc}
\begin{tabular}{lccccccccc}
\toprule
	& \multicolumn{3}{c}{100}
	& \multicolumn{3}{c}{300}
	& \multicolumn{3}{c}{400/500}
	\\
Alg: Data & \# Assess. & Train & Test & \# Assess. & Train & Test & \# Assess. & Train & Test \\
\midrule
AP-B: a6a & 3.35e+08 & 0.142 & 0.155 & 1.02e+09 & 0.131 & 0.157 & 1.69e+09 & 0.128 & 0.160 \\
QB: a6a & 3.57e+08 & 0.142 & 0.155 & 1.09e+09 & 0.131 & 0.157 & 1.83e+09 & 0.128 & 0.160 \\
AP-B: mnist & 1.26e+11 & 0.106 & 0.111 & 3.52e+11 & 0.057 & 0.064 & --- & --- & --- \\
QB: mnist & 1.36e+11 & 0.106 & 0.111 & 3.96e+11 & 0.057 & 0.064 & --- & --- & --- \\
AP-B: rcv1 & 1.73e+11 & 0.027 & 0.059 & 4.83e+11 & 0.005 & 0.047 & 6.15e+11 & 0.001 & 0.044 \\
QB: rcv1 & 1.85e+11 & 0.029 & 0.061 & 5.13e+11 & 0.004 & 0.047 & 6.58e+11 & 0.001 & 0.046 \\
AP-B: satimage & 1.98e+08 & 0.113 & 0.150 & 5.46e+08 & 0.070 & 0.121 & 8.64e+08 & 0.049 & 0.109 \\
QB: satimage & 2.20e+08 & 0.113 & 0.150 & 6.61e+08 & 0.070 & 0.121 & 1.11e+09 & 0.049 & 0.109 \\
AP-B: w4a & 3.92e+08 & 0.011 & 0.019 & 1.07e+09 & 0.006 & 0.018 & 1.69e+09 & 0.006 & 0.018 \\
QB: w4a & 4.64e+08 & 0.011 & 0.020 & 1.45e+09 & 0.006 & 0.018 & 2.41e+09 & 0.006 & 0.018 \\
\bottomrule
\end{tabular}
\end{sc}
\end{small}
\end{center}
\vskip -0.1in
\end{table*}

\subsection{Train and Test Error for LazyBoost and Weight Trimming}
\begin{table*}[ht]
\caption{Performance for A6A}
\label{tbl:perf-a6a}
\vskip 0.15in
\begin{center}
\begin{small}
\begin{sc}
\begin{tabular}{lccccccccc}
\toprule
	& \multicolumn{3}{c}{100}
	& \multicolumn{3}{c}{300}
	& \multicolumn{3}{c}{500}
	\\
 & \# Assess. & Train & Test & \# Assess. & Train & Test & \# Assess. & Train & Test \\
\midrule
AP LazyBoost (0.5) & 1.69e+08 & 0.145 & 0.156 & 5.11e+08 & 0.134 & 0.159 & 8.48e+08 & 0.129 & 0.160 \\
QB LazyBoost (0.5) & 1.80e+08 & 0.145 & 0.157 & 5.50e+08 & 0.137 & 0.158 & 9.22e+08 & 0.132 & 0.160 \\
AP LazyBoost (0.9) & 2.99e+08 & 0.141 & 0.156 & 9.07e+08 & 0.133 & 0.157 & 1.51e+09 & 0.130 & 0.159 \\
QB LazyBoost (0.9) & 3.18e+08 & 0.141 & 0.156 & 9.75e+08 & 0.133 & 0.157 & 1.64e+09 & 0.130 & 0.159 \\
AP Wt. Trim (0.9) & 2.45e+08 & 0.151 & 0.157 & 7.35e+08 & 0.151 & 0.157 & 1.23e+09 & 0.151 & 0.157 \\
QB Wt. Trim (0.9) & 2.49e+08 & 0.151 & 0.157 & 7.46e+08 & 0.151 & 0.157 & 1.24e+09 & 0.151 & 0.157 \\
AP Wt. Trim (0.99) & 3.16e+08 & 0.141 & 0.156 & 9.34e+08 & 0.132 & 0.157 & 1.54e+09 & 0.126 & 0.158 \\
QB Wt. Trim (0.99) & 3.28e+08 & 0.141 & 0.156 & 9.62e+08 & 0.132 & 0.157 & 1.58e+09 & 0.126 & 0.160 \\
\bottomrule
\end{tabular}
\end{sc}
\end{small}
\end{center}
\vskip -0.1in
\end{table*}

\begin{table*}[ht]
\caption{Performance for MNIST Digits}
\label{tbl:perf-mnist}
\vskip 0.15in
\begin{center}
\begin{small}
\begin{sc}
\begin{tabular}{lccccccccc}
\toprule
	& \multicolumn{3}{c}{100}
	& \multicolumn{3}{c}{200}
	& \multicolumn{3}{c}{300}
	\\
 & \# Assess. & Train & Test & \# Assess. & Train & Test & \# Assess. & Train & Test \\
\midrule
AP LazyBoost (0.5) & 6.65e+10 & 0.150 & 0.145 & 1.28e+11 & 0.098 & 0.098 & 1.87e+11 & 0.076 & 0.079 \\
QB LazyBoost (0.5) & 7.07e+10 & 0.150 & 0.145 & 1.39e+11 & 0.098 & 0.098 & 2.07e+11 & 0.076 & 0.079 \\
AP LazyBoost (0.9) & 1.17e+11 & 0.117 & 0.118 & 2.22e+11 & 0.079 & 0.085 & 3.20e+11 & 0.061 & 0.069 \\
QB LazyBoost (0.9) & 1.25e+11 & 0.117 & 0.118 & 2.43e+11 & 0.079 & 0.085 & 3.59e+11 & 0.061 & 0.069 \\
AP Wt. Trim (0.9) & 1.53e+11 & 0.901 & 0.901 & 3.07e+11 & 0.901 & 0.901 & 4.61e+11 & 0.901 & 0.901 \\
QB Wt. Trim (0.9) & 1.53e+11 & 0.900 & 0.901 & 3.07e+11 & 0.900 & 0.901 & 4.61e+11 & 0.900 & 0.901 \\
AP Wt. Trim (0.99) & 1.19e+11 & 0.117 & 0.124 & 2.21e+11 & 0.076 & 0.080 & 3.16e+11 & 0.062 & 0.068 \\
QB Wt. Trim (0.99) & 1.29e+11 & 0.115 & 0.117 & 2.37e+11 & 0.074 & 0.078 & 3.37e+11 & 0.056 & 0.061 \\
\bottomrule
\end{tabular}
\end{sc}
\end{small}
\end{center}
\vskip -0.1in
\end{table*}

\begin{table*}[ht]
\caption{Performance for RCV1}
\label{tbl:perf-rcv1}
\vskip 0.15in
\begin{center}
\begin{small}
\begin{sc}
\begin{tabular}{lccccccccc}
\toprule
	& \multicolumn{3}{c}{100}
	& \multicolumn{3}{c}{300}
	& \multicolumn{3}{c}{400}
	\\
 & \# Assess. & Train & Test & \# Assess. & Train & Test & \# Assess. & Train & Test \\
\midrule
AP LazyBoost (0.5) & 8.93e+10 & 0.029 & 0.061 & 2.48e+11 & 0.006 & 0.047 & 3.18e+11 & 0.002 & 0.046 \\
QB LazyBoost (0.5) & 9.06e+10 & 0.028 & 0.060 & 2.55e+11 & 0.005 & 0.048 & 3.29e+11 & 0.002 & 0.046 \\
AP LazyBoost (0.9) & 1.59e+11 & 0.027 & 0.058 & 4.35e+11 & 0.005 & 0.047 & 5.60e+11 & 0.002 & 0.045 \\
QB LazyBoost (0.9) & 1.64e+11 & 0.027 & 0.058 & 4.62e+11 & 0.004 & 0.047 & 5.93e+11 & 0.001 & 0.045 \\
AP Wt. Trim (0.9) & 1.19e+11 & 0.022 & 0.059 & 2.92e+11 & 0.003 & 0.047 & 3.65e+11 & 0.001 & 0.046 \\
QB Wt. Trim (0.9) & 1.22e+11 & 0.025 & 0.058 & 3.03e+11 & 0.003 & 0.047 & 3.79e+11 & 0.001 & 0.046 \\
AP Wt. Trim (0.99) & 1.62e+11 & 0.027 & 0.059 & 4.40e+11 & 0.004 & 0.047 & 5.61e+11 & 0.001 & 0.045 \\
QB Wt Trim (0.99) & 1.70e+11 & 0.027 & 0.059 & 4.60e+11 & 0.004 & 0.048 & 5.86e+11 & 0.001 & 0.046 \\
\bottomrule
\end{tabular}
\end{sc}
\end{small}
\end{center}
\vskip -0.1in
\end{table*}

\begin{table*}[ht]
\caption{Performance for SATIMAGE}
\label{tbl:perf-satimage}
\vskip 0.15in
\begin{center}
\begin{small}
\begin{sc}
\begin{tabular}{lccccccccc}
\toprule
	& \multicolumn{3}{c}{100}
	& \multicolumn{3}{c}{300}
	& \multicolumn{3}{c}{500}
	\\
 & \# Assess. & Train & Test & \# Assess. & Train & Test & \# Assess. & Train & Test \\
\midrule
AP LazyBoost (0.5) & 1.11e+08 & 0.133 & 0.152 & 3.22e+08 & 0.094 & 0.123 & 5.17e+08 & 0.073 & 0.115 \\
QB LazyBoost (0.5) & 1.23e+08 & 0.130 & 0.150 & 3.68e+08 & 0.090 & 0.129 & 6.11e+08 & 0.067 & 0.113 \\
AP LazyBoost (0.9) & 1.88e+08 & 0.114 & 0.128 & 5.13e+08 & 0.071 & 0.119 & 8.05e+08 & 0.050 & 0.110 \\
QB LazyBoost (0.9) & 2.06e+08 & 0.114 & 0.128 & 6.07e+08 & 0.071 & 0.119 & 1.01e+09 & 0.050 & 0.110 \\
AP Wt. Trim (0.9) & 2.51e+08 & 0.756 & 0.766 & 7.56e+08 & 0.756 & 0.766 & 1.26e+09 & 0.756 & 0.766 \\
QB Wt. Trim (0.9) & 2.51e+08 & 0.755 & 0.765 & 7.57e+08 & 0.755 & 0.765 & 1.26e+09 & 0.755 & 0.765 \\
AP Wt. Trim (0.99) & 1.80e+08 & 0.109 & 0.141 & 4.66e+08 & 0.066 & 0.121 & 7.01e+08 & 0.045 & 0.113 \\
QB Wt. Trim (0.99) & 1.89e+08 & 0.109 & 0.141 & 4.91e+08 & 0.066 & 0.121 & 7.39e+08 & 0.045 & 0.113 \\
\bottomrule
\end{tabular}
\end{sc}
\end{small}
\end{center}
\vskip -0.1in
\end{table*}

\begin{table*}[ht]
\caption{Performance for W4A}
\label{tbl:perf-w4a}
\vskip 0.15in
\begin{center}
\begin{small}
\begin{sc}
\begin{tabular}{lccccccccc}
\toprule
	& \multicolumn{3}{c}{100}
	& \multicolumn{3}{c}{300}
	& \multicolumn{3}{c}{500}
	\\
 & \# Assess. & Train & Test & \# Assess. & Train & Test & \# Assess. & Train & Test \\
\midrule
AP LazyBoost (0.5) & 2.00e+08 & 0.012 & 0.019 & 5.46e+08 & 0.008 & 0.018 & 8.61e+08 & 0.006 & 0.018 \\
QB LazyBoost (0.5) & 2.35e+08 & 0.012 & 0.019 & 7.35e+08 & 0.008 & 0.018 & 1.22e+09 & 0.006 & 0.018 \\
AP LazyBoost (0.9) & 3.48e+08 & 0.012 & 0.020 & 9.66e+08 & 0.007 & 0.018 & 1.52e+09 & 0.006 & 0.018 \\
QB LazyBoost (0.9) & 4.27e+08 & 0.012 & 0.020 & 1.32e+09 & 0.007 & 0.018 & 2.19e+09 & 0.006 & 0.018 \\
AP Wt. Trim (0.9) & 2.87e+08 & 0.016 & 0.021 & 8.41e+08 & 0.016 & 0.021 & 1.40e+09 & 0.016 & 0.021 \\
QB Wt. Trim (0.9) & 2.97e+08 & 0.016 & 0.021 & 8.63e+08 & 0.016 & 0.021 & 1.43e+09 & 0.016 & 0.021 \\
AP Wt. Trim (0.99) & 3.63e+08 & 0.012 & 0.020 & 9.44e+08 & 0.007 & 0.017 & 1.44e+09 & 0.006 & 0.018 \\
QB Wt. Trim (0.99) & 3.96e+08 & 0.012 & 0.020 & 1.01e+09 & 0.007 & 0.018 & 1.52e+09 & 0.006 & 0.018 \\
\bottomrule
\end{tabular}
\end{sc}
\end{small}
\end{center}
\vskip -0.1in
\end{table*}

\clearpage
\subsection {Different Tree Depths}
\begin{table*}[ht]
\caption{Different Tree Depths: Number of Assessments after 500 rounds}
\label{DiffDepths}
\vskip 0.15in
\begin{center}
\begin{small}
\begin{sc}
\begin{tabular}{llccccc}
\toprule
 &  & 1 & 2 & 3 & 4 & 5\\
\midrule
a6a & AP Boost & 6.40e+08 & 1.23e+09 & 1.69e+09 & 2.08e+09 & 2.44e+09 \\
a6a & Quick Boost & 6.66e+08 & 1.29e+09 & 1.83e+09 & 2.34e+09 & 2.89e+09 \\
\midrule
w4a & AP Boost & 8.71e+08 & 1.38e+09 & 1.69e+09 & 1.90e+09 & 2.12e+09 \\
w4a & Quick Boost & 9.10e+08 & 1.72e+09 & 2.41e+09 & 3.07e+09 & 3.60e+09 \\
\bottomrule
\end{tabular}
\end{sc}
\end{small}
\end{center}
\vskip -0.1in
\end{table*}

We also experimented with different tree depths, and found that \texttt{Adaptive-Pruning Boost} shows more dramatic gains in terms of total number of assessments when it uses deeper trees as weak learners.
We believe this is because of accumulated gains for training more nodes
in each tree. 
We have included an example of this in Table~\ref{DiffDepths}, where for two datasets (W4A, and A6A) we show experiments at depth 1 through 5. We report the total number of assessments used by \texttt{AdaBoost} (exact greedy-optimal decision trees) after 500 rounds.

\clearpage

\section{Information Gain}

\newcommand{\KL}[2]{\operatorname{KL}\left( {#1} \middle\| {#2} \right)}
\newcommand{\B}[1]{\mathcal{B}\left( {#1} \right)}
\newcommand{\BKL}[2]{\KL{\B{#1}}{\B{#2}}}

\newcommand{\Ztotal}{Z_{n}}
\newcommand{\Zleaf}{Z_{\rho}}
\newcommand{\Leaf}{\rho}
\newcommand{\Zseen}{Z_{u}}
\newcommand{\Zunseen}{Z_{\bar{u}}}
\newcommand{\Label}{y}
\newcommand{\Zlabel}{Z_{\Leaf}^{\Label}}
\newcommand{\Zlabelseen}{Z_{u}^{\Label}}
\newcommand{\Zlabelunseen}{Z_{\bar{u}}^{\Label}}
\newcommand{\Znotlabel}{Z_{\Leaf}^{\bar{\Label}}}
\newcommand{\Znotlabelseen}{Z_{u}^{\bar{\Label}}}
\newcommand{\Znotlabelunseen}{Z_{\bar{u}}^{\bar{\Label}}}
\newcommand{\Wtotal}{w}
\newcommand{\Wlabel}{w^{\Label}}
\newcommand{\Wnotlabel}{w^{\bar{\Label}}}

Notation reference:
\begin{itemize}
	\item $\Ztotal$ is the total weight of all $n$ training examples
	\item $\Zleaf$ is the weight of examples which reached
		some leaf $\Leaf$.
	\item $\Zseen$ and $\Zunseen$ are the seen and unseen weight for
		leaf $\rho$ (where $\rho$ should be clear from context),
		so $\Zseen + \Zunseen = \Zleaf$.
	\item $\Zlabel$ is the total weight for leaf $\Leaf$ with label $\Label$.
	\item $\Zlabelseen$ and $\Zlabelunseen$ are the seen and unseen
		weight for leaf $\rho$ with label $\Label$,
		so $\Zlabelseen + \Zlabelunseen = \Zlabel$.
	\item $\Znotlabel$ is the total weight for leaf $\Leaf$ with some
		label other than $\Label$,
		so $\Znotlabel = \Zleaf - \Zlabel$.
	\item $\Znotlabelseen$ and $\Znotlabelunseen$ are the seen and unseen
		weight for leaf $\rho$ with some label other than $\Label$,
		so $\Znotlabelseen + \Znotlabelunseen = \Znotlabel$.
	\item $\Wtotal$ is the total unseen weight for all leaves,
		so $\Wtotal = \sum_\Leaf \Zunseen$.
	\item $\Wlabel$ and $\Wnotlabel$ are the fraction of total unseen
		weight with and without label $\Label$,
		so $\Wlabel + \Wnotlabel = \Wtotal$.
\end{itemize}

The ``error'' term for Information Gain is the conditional entropy
of the leaves, written as follows.
\begin{align*}
    \epsilon_n := \sum_\Leaf \frac{ \Zleaf }{ \Ztotal } \left(
    	- \sum_\Label \frac{ \Zlabel }{ \Zleaf } \lg \frac{ \Zlabel }{ \Zleaf }
    	\right)
\quad
\Rightarrow
\quad
	\Ztotal \epsilon_n = \sum_\Leaf
	\overbrace{
	\left( - \sum_\Label \Zlabel \lg \frac{ \Zlabel }{ \Zleaf }
	\right)
	}^{\Zleaf \epsilon_\Leaf}
\end{align*}
\begin{align*}
	\Zleaf \epsilon_\Leaf
	&= - \sum_\Label \Zlabel \lg \frac{ \Zlabel }{ \Zleaf }
	= - \sum_\Label ( \Zlabelseen + \Zlabelunseen ) \lg \frac{
		\Zlabelseen + \Zlabelunseen
		}{
		\Zseen + \Zunseen
		}
\\
	&= \underbrace{\left( - \sum_\Label \Zlabelseen \lg \frac{\Zlabelseen}{\Zseen}
		\right)}_{\Zseen \epsilon_u}
		+ \underbrace{\left( - \sum_\Label \Zlabelunseen \lg \frac{\Zlabelunseen}{\Zunseen} 
		\right)}_{\Zunseen \epsilon_{\bar{u}}}
	+ \sum_\Label \Zlabel \BKL
		{ \frac{\Zlabelseen}{\Zlabel} }
		{ \frac{\Zseen}{\Zleaf} },
\end{align*}
where the final equality follows by Lemma~\ref{lemma-kl}, proved below.
The bounds on information gain thus ultimately depend on
$\Zunseen \epsilon_{\bar{u}}$ and on
the KL divergence term,
\begin{align}\label{eq:kl-term}
	\sum_\Label \Zlabel \BKL
		{ \frac{\Zlabelseen}{\Zlabel} }
		{ \frac{\Zseen}{\Zleaf} }
\end{align}
where $\KL{\cdot}{\cdot}$ is the Kullback-Liebler divergence
and $\B{\cdot}$ is a Bernoulli probability distribution.
\begin{align*}
	\BKL{p}{q}
	&= p \lg{\frac{p}{q}} + (1-p) \lg{\frac{1-p}{1-q}}
\end{align*}

Since $\Zunseen \epsilon_{\bar{u}} \ge 0$
and KL divergence are non-negative, a trivial lower bound is
\begin{align}
	\Zleaf \epsilon_\Leaf \ge \Zseen \epsilon_u
	= - \sum_\Label \Zlabelseen \lg \frac{\Zlabelseen}{\Zseen}.
\end{align}

It remains to prove an upper bound.
We upper bound the weight $\Zlabel$ of KL divergence
as $\Zlabel \le \Zlabelseen + \Wlabel$.
Below, we prove the following upper bound on the KL divergence
in Eq.~\ref{eq:kl-term}.

\begin{lemma}[KL Upper Bound]\label{thm:klub}
	For any individual leaf $\rho$ and label $\Label$, we have
	\begin{align*}
	    \BKL{ \frac{\Zlabelseen}{\Zlabel} }{ \frac{\Zseen}{\Zleaf} }
         \le \lg{\frac{\Zseen + \Wtotal}{\Zlabelseen}}.
\end{align*}
\end{lemma}

In order to complete our upper bound, we note that
$\Zunseen \epsilon_{\bar{u}}$
is simply the unassessed weight $\Zunseen$ times the label
entropy for the unassessed weight, and with $|Y|$ total labels
the label entropy is upper bounded as $\lg |Y|$.
This yields the following bounds on the conditional entropy term
for Information Gain.
\begin{align}
	\sum_\Leaf \Zseen \epsilon_u
	\le \Ztotal \epsilon_n \le
	\sum_\Leaf \left[
		\Zseen \epsilon_u
		+ w \lg |Y|
		+ \sum_y \left(
			(\Zlabelseen + \Wlabel) \lg{\frac{\Zseen + \Wtotal}{\Zlabelseen}}
		\right)
	\right]
\end{align}

Our proofs follow.

\begin{proof}[Proof of Lemma~\ref{thm:klub}]
We bound the KL divergence using the Reyni divergence and by
bounding the two Bernoulli probability ratios.
Our probabilities are
\begin{align}
\left(\frac{\Zlabelseen}{\Zlabel}, 1-\frac{\Zlabelseen}{\Zlabel}\right) = \left(\frac{\Zlabelseen}{\Zlabelseen + \Zlabelunseen}, \frac{\Zlabelunseen}{\Zlabelseen + \Zlabelunseen}\right)
\end{align}
and
\begin{align}
\left(\frac{\Zseen}{\Zleaf}, 1-\frac{\Zseen}{\Zleaf}\right) = \left(\frac{\Zseen}{\Zseen + \Zunseen}, \frac{\Zunseen}{\Zseen + \Zunseen}\right)
\end{align}.

Our two ratio are upper bounded as follows

\begin{align}
\frac{
	\left( \frac{\Zlabelseen}{\Zlabelseen + \Zlabelunseen} \right)
}{
	\left( \frac{\Zseen}{\Zseen + \Zunseen} \right)
} 
    = \frac{\Zlabelseen}{\Zlabelseen + \Zlabelunseen} \times \frac{\Zseen + \Zunseen}{\Zseen} 
    \le \frac{\Zlabelseen}{\Zlabelseen} \times \frac{\Zseen + \Wtotal}{\Zseen}
    \le \frac{\Zseen + \Wtotal}{\Zseen}
\end{align}

and 
\begin{align}
\frac{
	\left( \frac{\Zlabelunseen}{\Zlabelseen + \Zlabelunseen} \right)
}{
	\left( \frac{\Zunseen}{\Zseen + \Zunseen} \right)
}
&= \frac{\Zlabelunseen}{\Zlabelseen + \Zlabelunseen}\times \frac{\Zseen + \Zunseen}{\Zunseen}
= \frac{\Zlabelunseen}{\Zlabelseen + \Zlabelunseen}\times \frac{\Zseen + \Zunseen}{\Zlabelunseen+\Znotlabelunseen}
\le \frac{\Zlabelunseen}{\Zlabelseen + \Zlabelunseen}\times \frac{\Zseen + \Zunseen}{\Zlabelunseen}
\\
&\le \frac{\Zseen + \Wtotal}{\Zlabelseen}.
\end{align}
Since
\begin{align*}
\frac{\Zseen + \Wtotal}{\Zseen} \le \frac{\Zseen + \Wtotal}{\Zlabelseen},
\end{align*}
by the Reyni Divergence of $\infty$ order $D_{\infty}(\B{p}\|\B{q}) = \lg{\sup_i\frac{p_i}{q_i}}$ (i.e. the log of the maximum ratio of probabilities) we conclude that
\begin{align*}
\BKL{ \frac{\Zlabelseen}{\Zlabel} }{ \frac{\Zseen}{\Zleaf} }
         \le \lg{\frac{\Zseen + \Wtotal}{\Zlabelseen}}.
\end{align*}
\end{proof}

\begin{lemma}\label{lemma-kl}
For $a, b \ge 0$ and $\alpha$, $\beta > 0$
\begin{align}
    (a + b) \lg \frac{a + b}{\alpha + \beta} = a \lg\frac{a}{\alpha} + b\lg\frac{b}{\beta} - (a+b) \BKL
		{ \frac{a}{a+b} }
		{ \frac{\alpha}{\alpha + \beta} }.
\end{align}
\end{lemma}
\begin{proof}
\begin{align*}
       & (a + b) \lg \frac{a + b}{\alpha + \beta} \\
     = & a \lg \frac{a + b}{\alpha + \beta}  + b \lg \frac{a + b}{\alpha +     \beta} \\
     = & a \lg \frac{a}{\alpha} \frac{\alpha(a + b)}{a (\alpha + \beta)}  + b \lg \frac{b}{\beta}\frac{\beta (a + b)}{b (\alpha + \beta)} \\
     = & a \lg \frac{a}{\alpha}  +  b \lg \frac{b}{\beta} + (a+b)\left[ \frac{a}{a+b}\lg\frac{\alpha}{\alpha+\beta} \frac{a+b}{a} + \frac{b}{a+b}\lg\frac{\beta}{\alpha+\beta}\frac{a+b}{b}\right] \\
     = & a \lg \frac{a}{\alpha}  +  b \lg \frac{b}{\beta} + (a+b)\left[ \frac{a}{a+b}\lg\frac{\alpha/(\alpha+\beta)}{a/(a+b)} + \frac{b}{a+b}\lg\frac{\beta/\alpha+\beta}{b/a+b}\right] \\
     = & a \lg \frac{a}{\alpha}  +  b \lg \frac{b}{\beta} + (a+b)\left[  - \BKL
		{ \frac{a}{a+b} }
		{ \frac{\alpha}{\alpha + \beta} }\right]  \\
	= & a \lg\frac{a}{\alpha} + b\lg\frac{b}{\beta} - (a+b) \BKL
		{ \frac{a}{a+b} }
		{ \frac{\alpha}{\alpha + \beta} }
\end{align*}
\end{proof}

\end{appendices}

\end{document}